\documentclass{article}

 \usepackage[main, final]{neurips_2026}

\usepackage[utf8]{inputenc} 
\usepackage[T1]{fontenc}    
\usepackage{hyperref}       
\usepackage{url}            
\usepackage{booktabs}       
\usepackage{amsfonts}       
\usepackage{nicefrac}       
\usepackage{microtype}      
\usepackage{xcolor}         
\usepackage{amsmath}        
\usepackage{enumitem}       
\usepackage{graphicx}       
\usepackage{multirow}       
\usepackage{adjustbox}      
\usepackage{arydshln}       
\usepackage{algorithm}      
\usepackage{algorithmic}    
\usepackage{amsmath,amssymb,amsthm}  

\newtheorem{assumption}{Assumption}
\newtheorem{proposition}{Proposition}
\newtheorem{theorem}{Theorem}
\definecolor{my_magenta}{rgb}{0.906, 0.208, 0.525}
\definecolor{my_cyan}{rgb}{0.102, 0.698, 0.612}

\title{CLeaR: A Unified Framework for Resolving the Leakage–Degradation Dilemma in Style Transfer}

\author{
  Teng Zhou\\
  Zhejiang University \\
  \texttt{tengzhou@zju.edu.cn} \\
  \And
  Yunhao Chen\thanks{Equal Contribution, Corresponding Author}\\
  Fudan University \\
  \texttt{yhchen24@m.fudan.edu.cn} \\
}

\begin{document}

\maketitle

\vspace{-1.5em}

\begin{abstract}
Style transfer aims to render target content in the style of a reference image, but existing methods often suffer from content leakage, where objects, layouts, or semantics from the style reference appear in the generated output. Although prior data-driven and training-free methods can reduce leakage, they often face a leakage-degradation dilemma: stronger content suppression may weaken style fidelity, while richer style preservation may reintroduce unwanted reference content. We identify this dilemma across the full style-transfer pipeline, including feature separation, feature-space grounding, and diffusion generation. To address these issues, we propose CLeaR, a training-free framework for content-leakage-resistant style transfer. CLeaR first uses Orthogonal Subspace Projection to define content-reduced style targets in each vision foundation model (VFM) feature space. It then performs Ensemble Inversion, which optimizes a shared pixel-space style anchor satisfying style constraints across multiple VFMs. Finally, Energy-Guided Calibration maintains style alignment during diffusion sampling by steering the denoising trajectory toward the ensemble-defined style manifold. We further provide a theoretical analysis showing that the style-anchor estimation error decreases with the number of VFMs. Experiments on StyleBench demonstrate that CLeaR improves style alignment, reduces content leakage, and achieves better LLM-as-Judge evaluation compared with existing methods. The code is available at \href{https://github.com/0606zt/CLeaR}{https://github.com/0606zt/CLeaR}.
\end{abstract}


\section{Introduction}
\label{sec:introduction}

Given a style reference image and a target content condition, style transfer (ST) aims to render the target content in the reference style. Since early neural style transfer, this task has been closely tied to content-style (C-S) separation \cite{gatys2015neural}. With diffusion models becoming the dominant backbone for image generation \cite{rombach2022high,podell2024sdxl}, ST is increasingly formulated as conditional generation \cite{wang2023stylediffusion,zhang2023inversion,chung2024style,hertz2024style,wang2024instantstyle}. However, incomplete C-S disentanglement can cause content leakage, where objects, layouts, or semantics from the style reference appear in the output \cite{zhu2025less,rout2025rb,wang2024instantstyle}, reducing quality and user control.

To mitigate content leakage, existing methods commonly use visual representations from pretrained or learned encoders to define a style condition that is later injected into a diffusion generator. Data-driven methods train style-aware encoders, disentanglement modules, adapters, or projection layers on paired or curated stylization data \cite{gao2025styleshot,wang2025omnistyle2,zhang2024artbank,sohn2023styledrop,liu2023stylecrafter,ma2025scflow,qi2024deadiff,xing2025csgo}. Although effective in specific settings, they require additional training and are often constrained by the training distribution. Training-free methods instead manipulate pretrained VFM or diffusion features directly, such as subtracting content-related text features from image features \cite{wang2024instantstyle}, masking feature dimensions, shuffling spatial tokens, modifying attention keys and values, or guiding the denoising trajectory with reference-based objectives \cite{zhu2025less,wang2025styleadapter,roy2025duolora,chung2024style,hertz2024style,rout2025rb}. Despite these differences, both paradigms typically define style within a particular feature space and rely on the diffusion model to generate the final image from the resulting edited or guided condition.

Although these methods can reduce content leakage in some cases, they often lack stable control over the boundary between content removal and style preservation. As a result, existing methods tend to fail on different sides of the same dilemma \cite{rout2025rb}: methods that preserve rich style cues may also retain unwanted reference content, whereas methods that suppress reference content more strongly may remove important stylistic details. We refer to this as the leakage-degradation dilemma.

This dilemma is not caused by a single isolated step. It arises throughout the style-transfer pipeline, from separating C-S features, to grounding the style signal in a representation space, to injecting that signal during diffusion generation.

The first source is the \emph{feature separation stage}, where content-related and style-related components are separated within a reference representation. Existing operations such as subtraction, masking, patch manipulation, or attention editing can reduce content signals, but are mostly heuristic and may either leave residual content or discard style information correlated with content. To address this issue, we introduce Orthogonal Subspace Projection (OSP), which defines the style target geometrically in each VFM feature space. Given a content proxy, it removes the component of the reference feature aligned with the content direction and retains the orthogonal residual as the style feature.

The second source is the \emph{feature-space grounding stage}, where the extracted style signal is grounded in a particular representation space. Most existing methods rely on a single VFM feature space or a single learned representation, which is fragile because different VFMs encode style and content differently \cite{radford2021learning,kar2024brave,zhang2023tale,shen2026fine}. Directly aligning multiple VFM spaces is impractical because their dimensions, metrics, and semantic structures differ, while learning cross-space projectors would add training cost and risk domain overfitting. Instead, we make use of a simpler and more general observation: image pixel space provides a shared domain associated with all VFM feature spaces. This allows heterogeneous style constraints to be combined by optimizing an image itself, rather than by aligning feature spaces directly. Based on this idea, we introduce Ensemble Inversion (EI), which optimizes a learnable image tensor into a shared style anchor whose features satisfy the model-specific style constraints across all VFM branches.

The third source is the \emph{diffusion generation stage}, where the extracted style condition is injected into the generative process. Standard condition injectors such as IP-Adapter \cite{ye2023ip} are trained on ordinary image-text pairs rather than on content-reduced style anchors, so the denoising trajectory may drift from the intended style or reintroduce leaked content. To address this issue, we introduce Energy-Guided Calibration (EGC), a test-time guidance mechanism inspired by prior diffusion guidance methods \cite{yu2023freedom,bansal2023universal}. At selected denoising steps, it computes an ensemble feature-alignment energy between the current decoded image and the target style representation, and uses its gradient to calibrate the latent update toward the VFM-defined style manifold.

In summary, we introduce CLeaR (\underline{C}ontent \underline{Lea}kage \underline{R}esistant), a training-free framework for C-S disentangled style transfer. CLeaR defines content-reduced style targets through Orthogonal Subspace Projection, reconciles heterogeneous VFM targets through Ensemble Inversion, and maintains style alignment through Energy-Guided Calibration. Together, these components reduce content leakage at both the representation-extraction stage and the generation stage.

Our contributions are summarized as follows:
\begin{itemize}
    \item We propose CLeaR, a training-free framework that extracts a shared pixel-space style anchor via Orthogonal Subspace Projection and Ensemble Inversion, and maintains style alignment through Energy-Guided Calibration.
    \item We provide a theoretical analysis showing that the style-anchor estimation error decreases with the number of VFMs under limited cross-model dependence.
    \item Experiments on StyleBench \cite{gao2025styleshot} show improvements in style alignment, content leakage metrics, and LLM-as-Judge evaluation.
\end{itemize}


\section{Related Work}
\label{sec:related_work}

\paragraph{Style Transfer} Gatys et al. \cite{gatys2015neural} formalized style transfer by optimizing deep feature statistics. With diffusion models \cite{rombach2022high,podell2024sdxl} and conditioning adapters \cite{ye2023ip}, style transfer has increasingly become a conditional generation task, where the central challenge is still C-S separation \cite{wang2023stylediffusion}.

Existing methods can be broadly grouped into data-driven and training-free approaches. Data-driven methods learn style representations from paired or curated stylization data \cite{ma2025scflow,gao2025styleshot,wang2025omnistyle2,qi2024deadiff,zhang2024artbank,sohn2023styledrop,liu2023stylecrafter}, or train projection modules for C-S composition \cite{xing2025csgo}. These methods can be effective, but require additional training and are often constrained by the training distribution. Training-free methods instead manipulate pretrained VFM or diffusion features directly, including subtracting textual concepts from image embeddings \cite{wang2024instantstyle}, masking feature dimensions, shuffling spatial patches \cite{zhu2025less,wang2025styleadapter,ngweta2023simple}, modifying attention keys/values \cite{roy2025duolora,chung2024style,hertz2024style,he2026freestyle}, or adjusting the reverse diffusion trajectory through inversion or test-time optimization \cite{zhang2023inversion,hertz2024style,rout2025rb,zhou2025attention}. Though effective, these methods often face a leakage-degradation dilemma across the style-transfer pipeline: feature operations may leave residual content or remove useful style cues, single-space representations are fragile across VFMs, and style conditions drift during diffusion sampling. Based on this analysis, CLeaR combines Orthogonal Subspace Projection, Ensemble Inversion, and Energy-Guided Calibration to form a unified solution for content-leakage-resistant style transfer.

\paragraph{Model Inversion} Model inversion (MI) techniques originally emerged to analyze privacy risks by reconstructing training data from machine learning models \cite{fredrikson2014privacy, fredrikson2015model}. Early MI approaches primarily target classification networks, using gradient-based optimization and generative priors to reconstruct the original inputs from deep representations \cite{zhang2020secret, kahla2022label, struppek2022plug}. While traditional methods are generally tailored for classifiers, recent advancements \cite{chen2025leakyclip,chen2026side} have extended inversion paradigms to extract high-dimensional visual information from complex, large-scale VFMs successfully.

These inversion techniques provide a natural bridge between feature spaces and the pixel domain. And in our C-S separation process, the image pixel space serves as a universal grounding domain that corresponds to all VFM feature spaces. Leveraging this insight, we employ model inversion to combine heterogeneous representations.


\section{Methodology}

\begin{figure}[!htbp]
    \centering
    \includegraphics[width=1\linewidth]{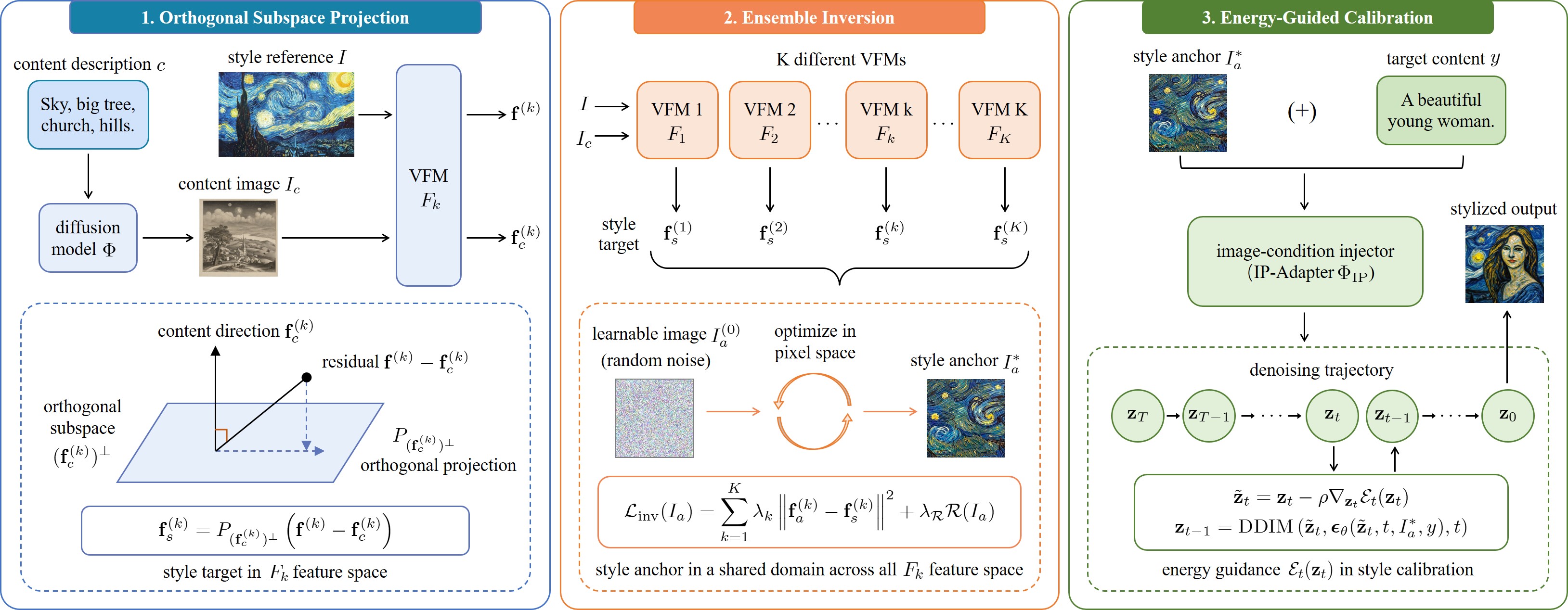}
    \vspace{-1.5em}
    \caption{Overview of the CLeaR framework. The three components target the three failure sources in the entire style transfer process: Orthogonal Subspace Projection (feature separation), Ensemble Inversion (feature-space grounding), and Energy-Guided Calibration (diffusion generation).}
    \label{fig:method}
\end{figure}

CLeaR contains three components that correspond to the three failure sources discussed in Sec.~\ref{sec:introduction}. Orthogonal Subspace Projection defines a content-reduced style target in each VFM feature space. Ensemble Inversion combines these model-specific targets by optimizing a shared pixel-space style anchor. Energy-Guided Calibration then keeps the generated image aligned with the target style during diffusion sampling. The overall pipeline is summarized in Fig.~\ref{fig:method} and Appendix~\ref{appendix:overall_pipe}.

\subsection{Orthogonal Subspace Projection}
\label{sec:ortho_proj}

Given a style reference image \(I\) and its content description \(c\) \cite{li2026toward}, our goal is to remove content-related information from the reference feature while retaining style information. Let \(F:\mathcal{I}\to\mathcal{H}\) be a pretrained VFM, and denote the reference feature as \(\mathbf{f}=F(I)\). To obtain a modality-matched content feature, we generate a content proxy image \(I_c=\Phi(c)\) using a text-to-image diffusion model \(\Phi\), and compute \(\mathbf{f}_c=F(I_c)\) \cite{chen2025leakyclip,gordon2024mismatch,kim2023misalign}.

We define the raw residual as
\begin{equation}
    \mathbf{d}=\mathbf{f}-\mathbf{f}_c.
\end{equation}
Since \(\mathbf{d}\) may still contain content-aligned information, we remove the component of \(\mathbf{d}\) along the content direction \(\mathbf{f}_c\). We seek a style feature \(\mathbf{f}_s=\mathbf{d}+\boldsymbol{\delta}\) that is orthogonal to \(\mathbf{f}_c\), while keeping the adjustment \(\boldsymbol{\delta}\) as small as possible:
\begin{equation}
    \min_{\boldsymbol{\delta}} \|\boldsymbol{\delta}\|^2
    \quad
    \text{s.t.}
    \quad
    \langle \mathbf{d}+\boldsymbol{\delta},\mathbf{f}_c\rangle=0.
\end{equation}
Solving them gives
\begin{equation}
\label{eq:ortho_proj_solution}
    \mathbf{f}_s
    =
    \mathbf{d}
    -
    \frac{\langle \mathbf{d},\mathbf{f}_c\rangle}{\|\mathbf{f}_c\|^2}\mathbf{f}_c
    =
    P_{\mathbf{f}_c^\perp}(\mathbf{d}),
\end{equation}
where \(P_{\mathbf{f}_c^\perp}\) denotes projection onto the orthogonal complement of \(\mathbf{f}_c\). Thus, \(\langle \mathbf{f}_s,\mathbf{f}_c\rangle=0\), and \(\mathbf{f}_s\) serves as a content-reduced style target in the VFM feature space.

\subsection{Ensemble Inversion}
\label{sec:ensb_inv}

The projection in Sec.~\ref{sec:ortho_proj} defines a style target within one VFM feature space. However, different VFMs capture different aspects of style, and their feature spaces cannot be directly aligned because they may have different dimensions and geometries. Instead, we make use of a simpler observation: image pixel space provides a shared domain associated with all VFM feature spaces.

Let \(\{F_k\}_{k=1}^{K}\) be a set of pretrained VFMs. For each model, we compute
\begin{equation}
    \mathbf{f}^{(k)}=F_k(I),
    \qquad
    \mathbf{f}_c^{(k)}=F_k(I_c),
\end{equation}
and obtain the model-specific style target
\begin{equation}
\label{eq:ensb_inv_k_solutions}
    \mathbf{f}_s^{(k)}
    =
    P_{(\mathbf{f}_c^{(k)})^\perp}
    \left(
    \mathbf{f}^{(k)}-\mathbf{f}_c^{(k)}
    \right).
\end{equation}

We then optimize a single image \(I_a\) whose features match these style targets across all VFM branches. Specifically, let
\begin{equation}
    \mathbf{f}_a^{(k)}=F_k(I_a).
\end{equation}
The ensemble inversion objective is
\begin{equation}
\label{eq:ensb_inv_objective}
    \mathcal{L}_{\mathrm{inv}}(I_a)
    =
    \sum_{k=1}^{K}
    \lambda_k
    \left\|
    \mathbf{f}_a^{(k)}-\mathbf{f}_s^{(k)}
    \right\|^2
    +
    \lambda_{\mathcal{R}}\mathcal{R}(I_a),
\end{equation}
where \(\lambda_k\) controls the contribution of each VFM and \(\mathcal{R}(\cdot)\) is a regularizer for spatial smoothness.

Starting from a randomly initialized image tensor \(I_a^{(0)}\), we update the anchor by gradient descent:
\begin{equation}
\label{eq:ensb_inv_grad}
    I_a^{(t+1)}
    =
    I_a^{(t)}
    -
    \eta
    \nabla_{I_a}
    \mathcal{L}_{\mathrm{inv}}
    \left(I_a^{(t)}\right),
\end{equation}
where \(\eta\) is the learning rate. After optimization, the resulting image \(I_a^*\) is used as the style anchor. This anchor provides a shared pixel-space representation whose VFM features match the content-reduced style targets across the ensemble.

\subsection{Energy-Guided Calibration}
\label{sec:style_calib}

The style anchor \(I_a^*\) is used as the image condition for diffusion generation. However, standard condition injectors such as IP-Adapter are trained on ordinary image-text pairs and may drift from the intended style during denoising. To reduce this drift, we introduce a test-time Energy-Guided Calibration step, following the idea of gradient-based diffusion guidance \cite{yu2023freedom,bansal2023universal}.

At denoising step \(t\), we first estimate the clean latent from the current noisy latent:
\begin{equation}
\label{eq:pred_clean_latent}
    \hat{\mathbf z}_0(\mathbf z_t)
    =
    \frac{
    \mathbf z_t-\sigma_t
    \boldsymbol{\epsilon}_\theta(\mathbf z_t,t,I_a^*,y)
    }{\alpha_t},
\end{equation}
where \(\alpha_t\) and \(\sigma_t\) are determined by the noise schedule. 
Using the same VFM ensemble as in Ensemble Inversion, we define the style-alignment energy as
\begin{equation}
\label{eq:style_calib_energy}
    \mathcal{E}_t(\mathbf z_t)
    =
    \sum_{k=1}^{K}
    \lambda_k
    \left\|
    F_k(D(\hat{\mathbf z}_0(\mathbf z_t)),)
    -
    \mathbf f_s^{(k)}
    \right\|^2.
\end{equation}
where \(D\) is the VAE decoder. A lower value of \(\mathcal{E}_t\) indicates stronger alignment with the target style representation.

We then use the energy gradient to calibrate the current latent:
\begin{equation}
\label{eq:style_calib_latent}
    \tilde{\mathbf z}_t
    =
    \mathbf z_t
    -
    \rho
    \nabla_{\mathbf z_t}
    \mathcal{E}_t(\mathbf z_t),
\end{equation}
where \(\rho\) controls the calibration strength. The calibrated latent is passed to DDIM \cite{song2021denoising} update:
\begin{equation}
\label{eq:style_calib_ddim}
    \mathbf z_{t-1}
    =
    \mathrm{DDIM}
    \left(
    \tilde{\mathbf z}_t,
    \boldsymbol{\epsilon}_\theta(\tilde{\mathbf z}_t,t,I_a^*,y),
    t
    \right).
\end{equation}
This calibration requires no retraining and improves style alignment by steering the sampling trajectory toward the ensemble-defined style target.

\subsection{Theoretical Analysis}

We analyze how the accuracy of ensemble inversion depends on the number of feature extractors.
Let $x^\star \in \mathbb{R}^d$ denote the ideal content-free style anchor, and let
$F_1,\dots,F_K$ be $K$ feature extractors.
For each model $k$, let the target style feature be $s_k$.
We consider the estimator
\begin{equation}
\hat x_K
=
\arg\min_x
\left\{
\frac{1}{K}\sum_{k=1}^K \|F_k(x)-s_k\|^2
+
\lambda \|x-x^\star\|^2
\right\},
\label{eq:theory_obj}
\end{equation}
where $\lambda>0$ is a regularization coefficient.

Let
\begin{equation}
J_k := \nabla F_k(x^\star),
\qquad
H_K := \frac{1}{K}\sum_{k=1}^K J_k^\top J_k.
\label{eq:theory_HK}
\end{equation}
We write the model-specific target error as
\begin{equation}
s_k = F_k(x^\star)+\varepsilon_k,
\qquad
\mu_k := \mathbb{E}[\varepsilon_k],
\qquad
\tilde\varepsilon_k := \varepsilon_k-\mu_k.
\label{eq:theory_noise_decomp}
\end{equation}

\begin{assumption}[Diversity and local identifiability]
\label{ass:main_div_ident}
The centered residuals have bounded second moments and limited average cross-model dependence:
\begin{equation}
\|\mathbb{E}[\tilde\varepsilon_k\tilde\varepsilon_k^\top]\|_{\mathrm{op}} \le \sigma^2,
\qquad
\|\mathbb{E}[\tilde\varepsilon_k\tilde\varepsilon_j^\top]\|_{\mathrm{op}}
\le
\rho_K \sigma^2
\quad (k\neq j),
\label{eq:theory_dep}
\end{equation}
where $\rho_K \in [0,1]$ is an ensemble dependence coefficient.
In addition, the ensemble is locally identifiable:
\begin{equation}
H_K \succeq m_K I
\qquad \text{for some } m_K>0.
\label{eq:theory_ident}
\end{equation}
\end{assumption}

Assumption~\ref{ass:main_div_ident} is motivated by two standard ideas.
First, heterogeneous vision encoders can provide complementary feature views rather than identical errors, making cross-model dependence, rather than exact equality, the key quantity to control.
Second, local recovery in inverse problems is typically tied to a nondegenerate Jacobian or sensitivity matrix.
These perspectives are consistent with recent evidence on complementary encoder biases and ensemble diversity, as well as standard identifiability analyses in inverse problems \cite{kar2024brave,wood2023unified,rouchier2018solving,cintron2020sensitivity,wang2026embedding}.

\begin{theorem}[Dependence-controlled scaling with ensemble size] Under Assumption~\ref{ass:main_div_ident} and the usual local regularity conditions deferred to Appendix~\ref{appendix:proof_multi_model}, the expected inversion error satisfies
\label{thm:dependence_scaling}
\begin{equation}
\mathbb{E}\|\hat x_K-x^\star\|^2
\lesssim
\frac{d\,\bar M_K\,\sigma^2}{(m_K+\lambda)^2}
\left(
\rho_K+\frac{1-\rho_K}{K}
\right),
\label{eq:theory_main_bound}
\end{equation}
where
\begin{equation}
\bar M_K := \frac{1}{K}\sum_{k=1}^K \|J_k\|_{\mathrm{op}}^2.
\label{eq:theory_Mbar}
\end{equation}
\end{theorem}
In particular, the gain from enlarging the ensemble is governed by the cross-model dependence coefficient $\rho_K$.
When the centered residuals are weakly correlated across models, that is, when $\rho_K$ is small, the dominant term decreases approximately as $1/K$. 

\paragraph{Empirical connection} The bound in Eq.~\ref{eq:theory_main_bound} predicts that the benefit of increasing \(K\) depends on the complementarity of the VFM residuals. This prediction is consistent with the VFM composition ablation in Fig.~\ref{fig:vfm_abl}: individual VFMs emphasize different aspects of C-S disentanglement, while increasing the number of VFMs improves both SA and CL. This empirical trend supports the interpretation that Ensemble Inversion benefits from complementary VFM constraints.


\section{Experiments}

\subsection{Settings}
\label{sec:setting}

\paragraph{Evaluation Dataset} We evaluate our method on StyleBench \cite{gao2025styleshot}, a recent dataset containing 40 content images and 490 style images across 73 distinct styles. For each style, we select one representative image (\(I\)) and generate its content description (\(c\)) using Qwen3 \cite{yang2025qwen3}. For content, we use the textual prompts (\(y\)) corresponding to the 40 content images as generation conditions. This yields 2,920 style-content pairs for our experiments.

\paragraph{Implementation Details} For each style-content pair, we generate 10 images over 25 random seeds. We use five VFMs in CLeaR: CLIP \cite{radford2021learning,schlarmann2024robust}, CSD-CLIP \cite{somepalli2024measuring}, DINO \cite{simeoni2025dinov3}, VGG \cite{simonyan2014very}, and Inception \cite{szegedy2016rethinking}. In Ensemble Inversion, all VFMs are equally weighted with \(\lambda_k=0.2\), and the regularizer \(\mathcal{R}(\cdot)\) is total variation (TV) \cite{rudin1992nonlinear} with \(\lambda_{\mathcal{R}}=0.05\). We optimize for 300 iterations with learning rate \(\eta=0.01\). For Energy-Guided Calibration, guidance is applied only in the refinement stage \cite{yu2023freedom}, i.e., \(t:T/10\to0\), for 5 iterations with correction strength \(\rho=0.01\). Further details are in Sec.~\ref{sec:ablation}.

\paragraph{Evaluation Metrics} We evaluate seven metrics across four aspects:

\begin{itemize}[leftmargin=*,noitemsep,topsep=0pt]
    \item \underline{\textbf{Style Alignment}}: We measure style alignment using CSD-CLIP cosine similarity and VGG Gram-matrix style loss \cite{gatys2015neural} between the generated images and the style references.

    \item \underline{\textbf{Content Alignment}}: We measure content alignment using CLIP cosine similarity between the generated images and the target content prompts.

    \item \underline{\textbf{Content Leakage}}: We measure content leakage using DINO similarity and KID \cite{binkowski2018demystifying} between the generated images and the content images. We also use Qwen3 to rate leaked reference-content elements on a 1--5 scale.

    \item \underline{\textbf{Aesthetic Quality}}: We measure aesthetic quality using the CLIP-Aesthetic score \cite{schuhmann2022laion}.
\end{itemize}

\paragraph{Baselines} We compare CLeaR with eight recent methods: Attention Distillation \cite{zhou2025attention}, CSGO \cite{xing2025csgo}, DEADiff \cite{qi2024deadiff}, InstantStyle \cite{wang2024instantstyle}, MaskST \cite{zhu2025less}, RB-Modulation \cite{rout2025rb}, StyleAligned \cite{hertz2024style}, and StyleShot \cite{gao2025styleshot}. Since style transfer methods vary in input modalities, and our framework uses a style image and a content text prompt, we restrict our comparison to baselines that adopt the same modality combination to ensure a fair evaluation.

\subsection{Comparisons}
\label{sec:comparison}

\paragraph{Visualization of Style Anchors}

\begin{figure}[!htbp]
    \centering
    \includegraphics[width=\linewidth]{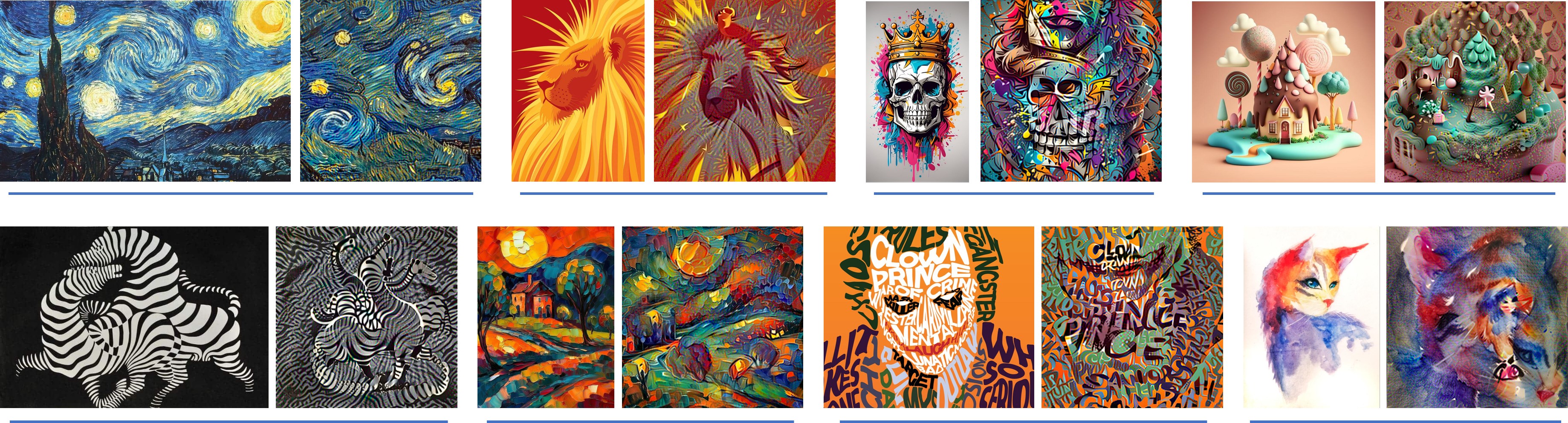}
    \vspace{-1.5em}
    \caption{Visualization of style anchor images extracted by Ensemble Inversion. For each example, the left image is the reference and the right image is the style anchor. The anchor successfully achieves C-S disentanglement, retaining stylistic elements (e.g., color, texture, brushstroke) while suppressing content semantics.}
    \label{fig:style_anchor}
\end{figure}

Fig.~\ref{fig:style_anchor} visualizes the key intermediate output of our method, the style anchor image extracted via Ensemble Inversion. For each example, the left image is the original style reference \(I\), and the right image is the corresponding style anchor \(I_a^*\). The style anchor largely removes the content objects of the reference while faithfully preserving style-related representations. For instance, in the first example (row 1, col 1), the anchor removes the large tree, hills, and church steeple, retaining only the starry brushstrokes. In the seventh example (row 2, col 3), the clown figure is removed, leaving only the typographic style. These indicate that our method effectively disentangles content elements from stylistic attributes. More examples are provided in Appendix~\ref{appendix:style_anchor}.

\begin{figure}[!htbp]
    \centering
    \includegraphics[width=\linewidth]{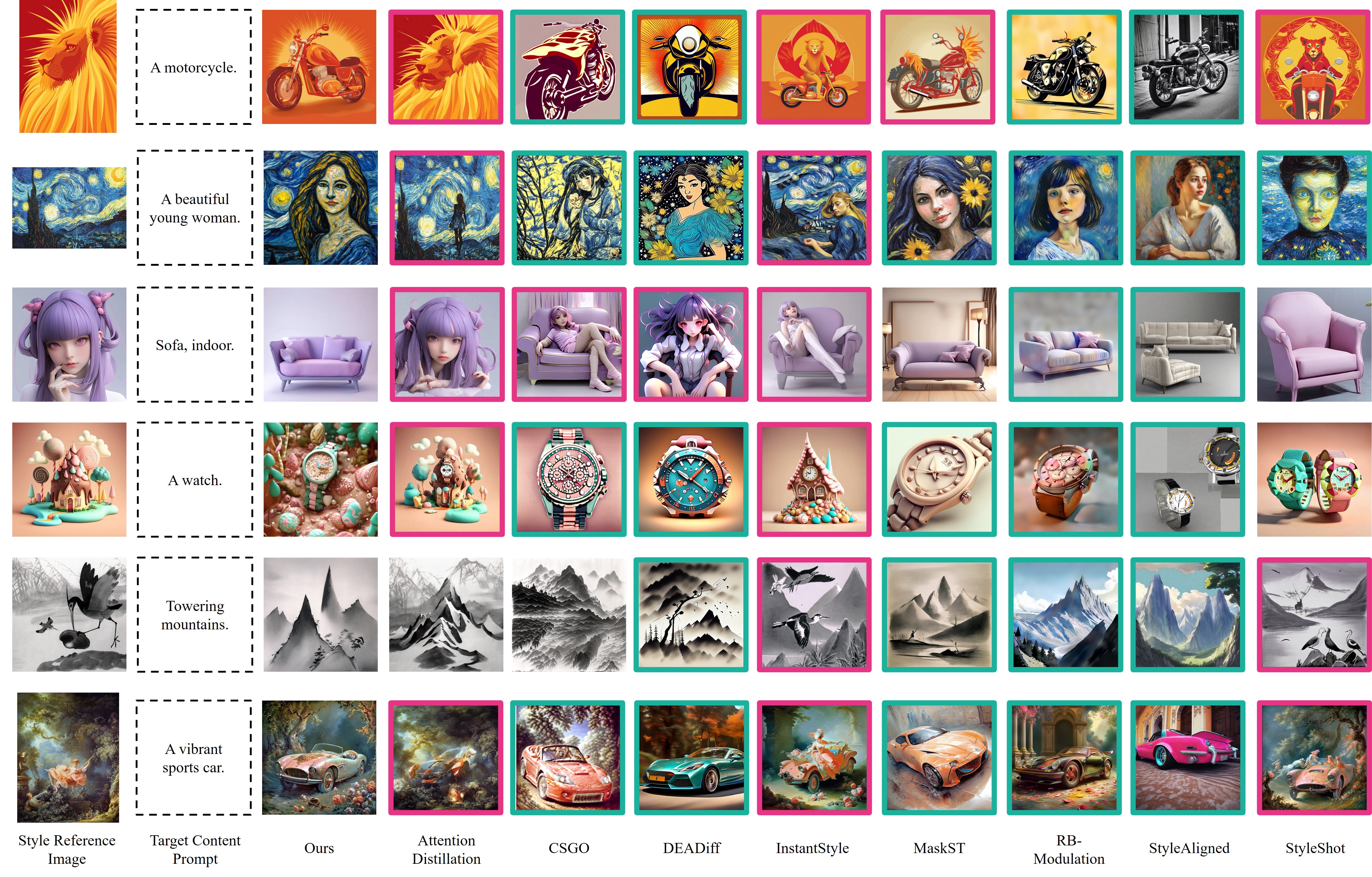}
    \vspace{-1.5em}
    \caption{Qualitative comparison results. Existing methods suffer from either \textcolor{my_magenta}{content leakage} or \textcolor{my_cyan}{style degradation}, while our method achieves robust stylization without either problem.}
    \label{fig:qual_comp}
\end{figure}

\paragraph{Qualitative Results} Fig.~\ref{fig:qual_comp} shows the qualitative comparisons, with magenta boxes marking content leakage and cyan boxes marking style degradation. Existing methods often fail on one side of the leakage-degradation dilemma: data-driven methods such as StyleShot, DEADiff, and CSGO show limited style alignment on out-of-distribution references; feature-manipulation methods such as InstantStyle and MaskST either retain residual content or suppress style details; and trajectory-level methods such as RB-Modulation and StyleAligned can depend on unavailable style-name priors. Attention Distillation preserves texture but often introduces reference content. In contrast, CLeaR achieves stronger style transfer with less visible content leakage. More results are in Appendix~\ref{appendix:qual_comp}.

\begin{table}[!htbp]
    \fontsize{12pt}{17pt}\selectfont
    \centering
    \caption{Quantitative comparison of CLeaR and other baselines, with the \textbf{best} and \underline{second-best} scores marked accordingly. CLeaR outperforms others across most aspects, especially in style alignment, content leakage, and content alignment.}
    \label{tab:quant_comp}
    \begin{adjustbox}{width=\linewidth}
	\begin{tabular}{cccccccccc}
        \hline
        \multirow{2}{*}{} & \multicolumn{2}{c}{Style \; Alignment} & Content Alignment & \multicolumn{3}{c}{Content Leakage} & Aesthetic Quality \\
        \cmidrule(lr){2-3} \cmidrule(lr){4-4} \cmidrule(lr){5-7} \cmidrule(lr){8-8}  
        & CSD\(\uparrow\) & Style Loss\(\downarrow\) & CLIP\(\uparrow\) & DINO\(\downarrow\) & KID\(\uparrow\) & AI Scoring\(\uparrow\) & CLIP-Aesthetic\(\uparrow\) \\
        \hline
        Attention Distillation & \underline{0.630} & 0.106 & 0.179 & 0.322 & 0.112 & 1.84 & 5.861 \\
        CSGO & 0.438 & 0.302 & 0.215 & \underline{0.168} & \underline{0.237} & 3.65 & 5.912 \\
        DEADiff & 0.369 & 0.259 & 0.225 & 0.209 & 0.233 & 3.64 & 6.349 \\
        InstantStyle & 0.492 & 0.119 & 0.221 & 0.263 & 0.092 & 2.75 & 6.507 \\
        MaskST & 0.416 & 0.139 & 0.225 & 0.197 & 0.103 & 3.45 & \underline{6.605} \\
        RB-Modulation & 0.443 & 0.177 & 0.234 & 0.213 & 0.111 & 3.82 & \textbf{6.677} \\
        StyleAligned & 0.206 & 0.252 & \textbf{0.241} & 0.177 & 0.120 & \underline{3.85} & 6.573 \\
        StyleShot & 0.505 & \underline{0.097} & 0.202 & 0.172 & 0.174 & 3.72 & 6.164 \\
        Ours & \textbf{0.682} & \textbf{0.034} & \underline{0.238} & \textbf{0.066} & \textbf{0.358} & \textbf{4.16} & 6.445 \\
        \hline
	\end{tabular}
    \end{adjustbox}
    \vspace{-1em}
\end{table}

\paragraph{Quantitative Results} Tab.~\ref{tab:quant_comp} reports the quantitative comparison. CLeaR achieves the best performance on most metrics, with clear gains in the two key aspects of C-S disentanglement: style alignment and content leakage suppression. Compared with the strongest baseline, CLeaR improves CSD by 8.3\%, Style Loss by 64.9\%, DINO by 60.7\%, KID by 51.1\%, and AI-Scoring by 8.05\%, while maintaining competitive content alignment with the second-best CLIP score. Its slightly lower aesthetic score may result from the inversion process, which prioritizes disentanglement and structural preservation over aesthetic smoothing. We further validate these conclusions in Appendix~\ref{appendix:heldout_quant_comp} using held-out VFMs and AI evaluators, which show consistent trends and confirm that the reported gains are not biased by the closed-loop evaluation setup.

\subsection{Ablation Studies}
\label{sec:ablation}

Our ablations examine three key factors in CLeaR: the content description of the style reference in OSP, the VFM composition in EI, and the guidance stage in EGC. For clarity, we report two aggregate metrics: style alignment (SA\(\uparrow\)) and content leakage suppression (CL\(\uparrow\)). We first normalize each metric to \([0,1]\) across all compared settings. SA is the average of normalized CSD and \(1-\text{normalized Style Loss}\), since higher CSD and lower Style Loss indicate better style alignment. CL is the average of \(1-\text{normalized DINO}\), normalized KID, and normalized AI score, since lower DINO similarity, higher KID, and higher AI score indicate less leakage.

\vspace{0.3em}

\begin{figure}[H]
    \centering
    \begin{minipage}{0.38\textwidth}
        \includegraphics[width=\linewidth]{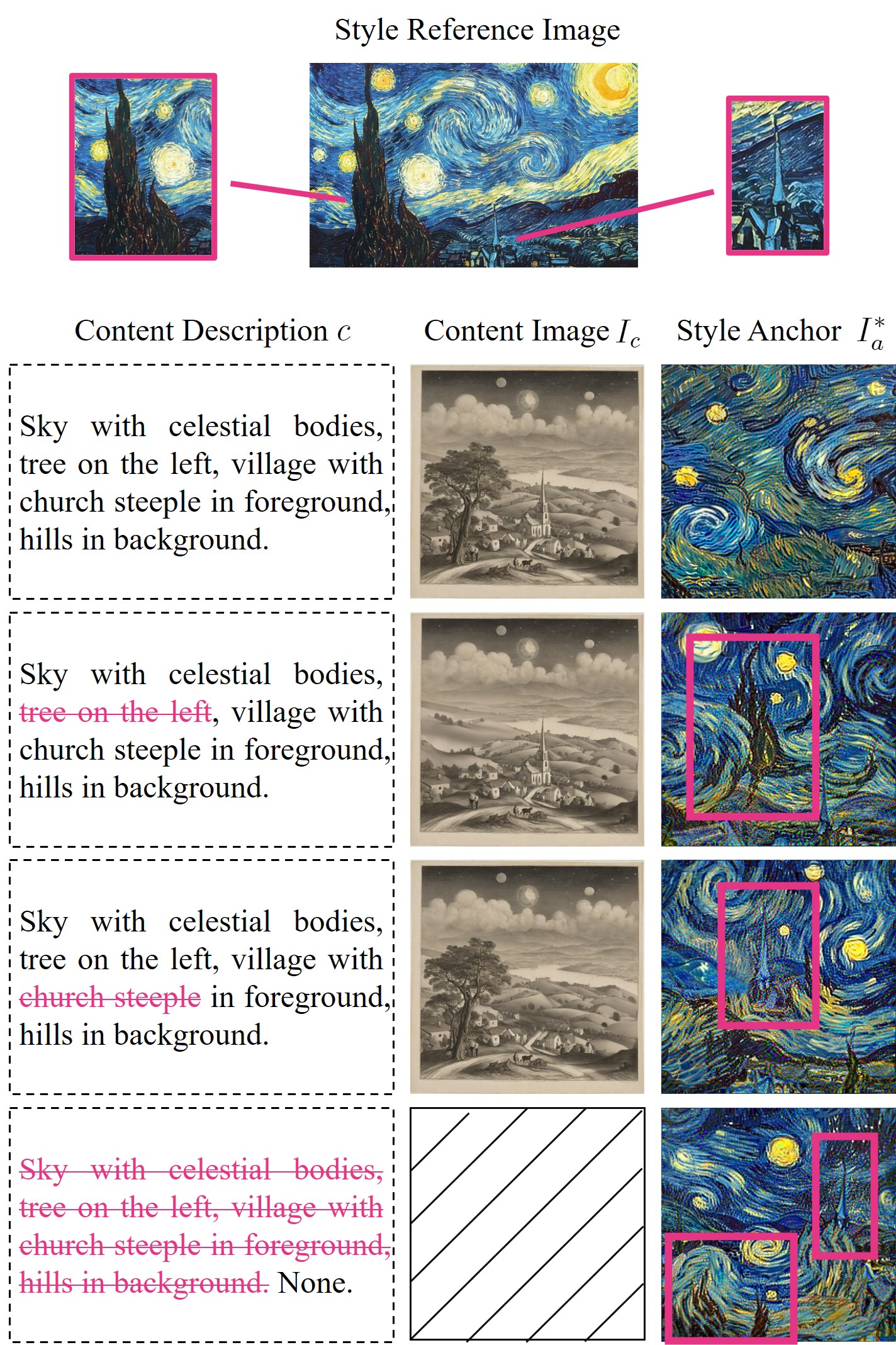}
        \vspace{-0.3em}
        \caption{Style anchor extraction with varying content descriptions \(c\). For a given reference, different \(c\) yield different content images \(I_c\) and style anchors \(I_a^*\). Removing an object from \(c\) preserves that object in \(I_a^*\), showing flexible disentanglement control.}
        \label{fig:style_content_abl}
    \end{minipage}
    \hfill
    \begin{minipage}{0.58\textwidth}
        \paragraph{Content Description of Style Reference} The orthogonal projection in Eq.~\ref{eq:ortho_proj_solution} requires a content description \(c\) of the style reference image \(I\) to generate a content image \(I_c = \Phi(c)\). This description determines which content elements are treated as “content” to be subtracted from the style feature. To examine the flexibility of our C‑S disentanglement, we manually vary \(c\) by removing specific objects (e.g., “tree”, “church steeple”) and observe the resulting style anchor \(I_a^*\) extracted via Ensemble Inversion.

        \vspace{0.5em}

        As shown in Fig.~\ref{fig:style_content_abl}, when \(c\) contains all content elements (row 1), \(I_c\) includes the full object set, and the extracted anchor retains only the starry style brushstrokes. Removing “tree” from \(c\) (row 2) causes the tree to reappear in \(I_a^*\), while the church steeple is suppressed; removing “church steeple” (row 3) yields the opposite pattern. When no content is explicitly described (row 4), the inversion lacks guidance and retains many content artifacts. These results demonstrate that our framework allows fine‑grained control over which content elements are disentangled, simply by editing the natural-language description \(c\), without retraining any component.

        \vspace{1.2em}

        \paragraph{VFM Selection} To assess the effect of VFM composition, we evaluate Ensemble Inversion with \(k=1,3,5\) models selected from the five VFMs used in our main experiments.

        \vspace{0.5em}

        As shown in Fig.~\ref{fig:vfm_abl}, CLIP, DINO, and Inception mainly help suppress content leakage, while VGG and CSD-CLIP better preserve fine-grained style patterns. Increasing the number of VFMs consistently improves both SA and CL, showing that complementary VFM features lead to stronger C-S disentanglement.
    \end{minipage}
\end{figure}

\begin{figure}[!htbp]
    \centering
    \includegraphics[width=\linewidth]{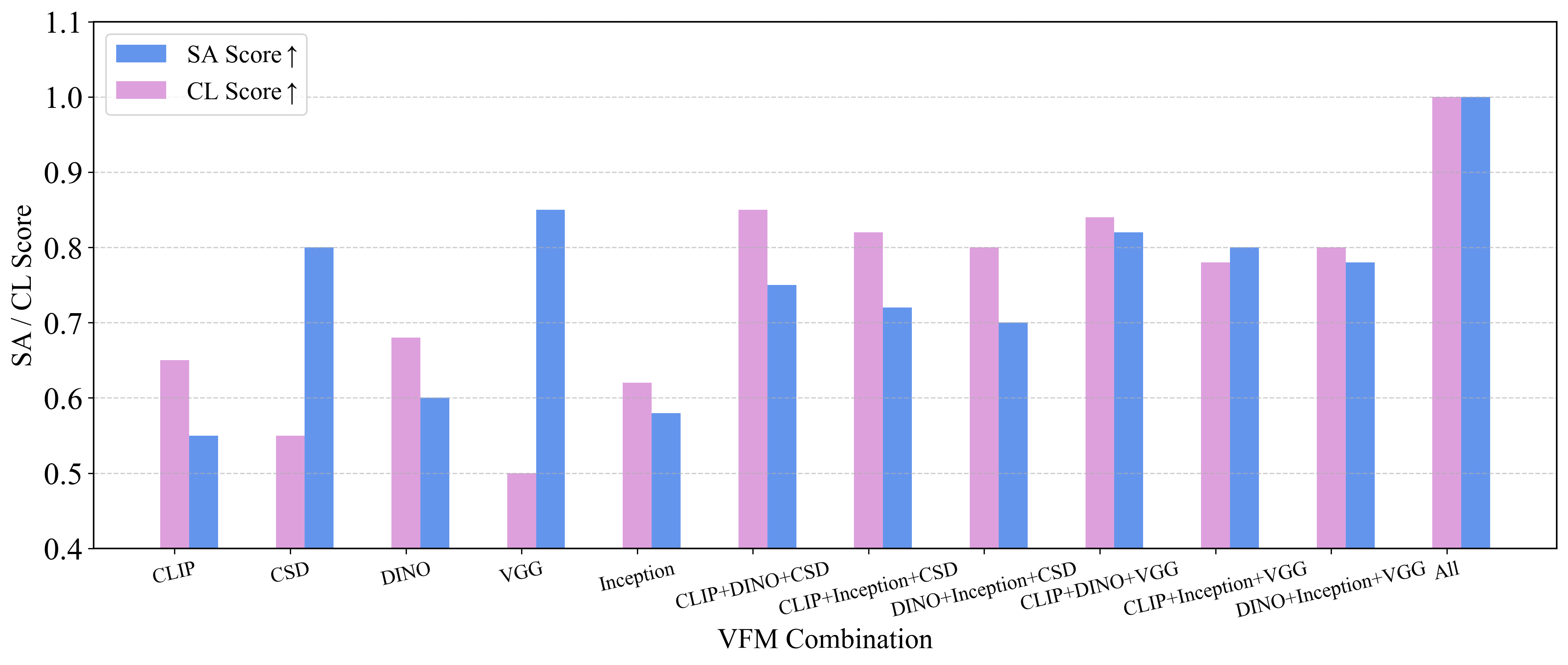}
    \vspace{-1.3em}
    \caption{Ablation of VFM compositions. CLIP, DINO, and Inception excel at suppressing content leakage, while VGG and CSD-CLIP better preserve fine-grained style patterns. Increasing the number of models improves both metrics, with the full ensemble of five achieving the best performance.}
    \label{fig:vfm_abl}
\end{figure}

\begin{figure}
    \centering
    \includegraphics[width=\linewidth]{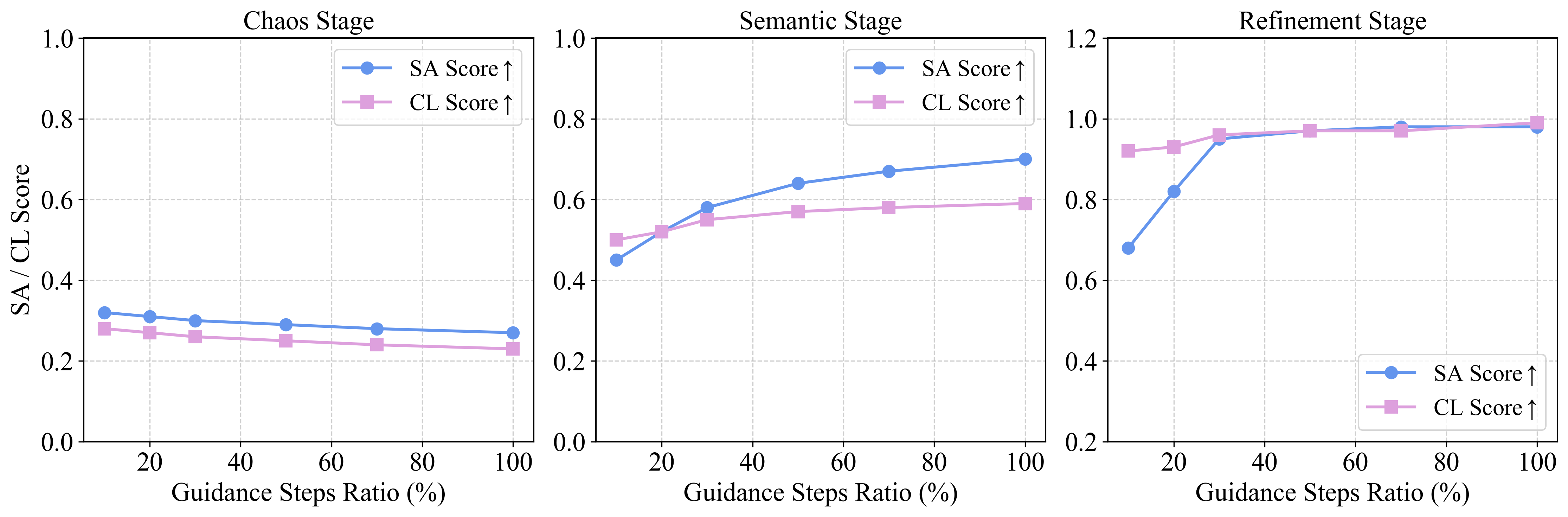}
    \vspace{-1.3em}
    \caption{Effect of guidance stage in Energy-Guided Calibration. We compare applying guidance in the chaotic, semantic, and refinement stages, with the x-axis indicating the guided proportion of each stage. Guidance is ineffective in the chaotic stage, while in the semantic and refinement stages it improves SA with little change in CL. The best efficiency-performance trade-off is achieved by guiding the final 20\% of the refinement stage, equivalent to the final \(T/10\) steps.}
    \label{fig:calib_stage_abl}
    \vspace{-1.3em}
\end{figure}

\vspace{-1.5em}

\paragraph{Calibration Stage} We apply Energy-Guided Calibration only in the final \(T/10\) timesteps by default for efficiency. To study the effect of guidance timing, we divide sampling into three intervals: the chaotic stage \((t:T\to4T/5)\), where the image is mostly noise; the semantic stage \((t:4T/5\to T/2)\), where the global structure emerges; and the refinement stage \((t:T/2\to0)\), where local textures and details are formed. For each interval, we vary the guided timestep ratio from 10\% to 100\%.

As shown in Fig.~\ref{fig:calib_stage_abl}, guidance is ineffective in the chaotic stage. In the semantic and refinement stages, SA improves as more guidance steps are used, while CL remains relatively stable, indicating that calibration mainly enhances style alignment without increasing content leakage. The best efficiency-performance trade-off is achieved by guiding 20\% of the refinement stage, equivalent to the final \(T/10\) timesteps, which we use as the default setting.


\section{Conclusion}

We presented CLeaR, a training-free framework for content-leakage-resistant style transfer. We identified that the leakage-degradation dilemma arises across the entire style transfer process: from separating C-S features, to grounding the style signal in a representation space, to injecting it during diffusion generation. CLeaR addresses these three sources with Orthogonal Subspace Projection for content-reduced style targets, Ensemble Inversion for multi-VFM style anchor extraction, and Energy-Guided Calibration for style-preserving diffusion sampling. Our theoretical analysis supports the benefit of using multiple VFMs, and experiments on StyleBench show improved style alignment and content leakage suppression. These results demonstrate the effectiveness of addressing C-S disentanglement jointly across feature separation, style representation, and generation.

\paragraph{Limitations} Despite its effectiveness, CLeaR has several limitations. First, it introduces additional test-time computation via Ensemble Inversion and Energy-Guided Calibration. While the measured runtime is comparable to several training-free baselines, it is still slower than simpler feed-forward methods. Second, as analyzed in Appendix~\ref{appendix:failure_mode}, ambiguous content descriptions, degenerate anchors, and adapter misinterpretation of inversion noise can still lead to leakage and artifacts. These cases represent an important direction for future work.


{\small
\bibliographystyle{plain}
\bibliography{mybib}
}


\newpage
\appendix

\setcounter{page}{1}
\setcounter{figure}{0}
\setcounter{table}{0}
\renewcommand\thefigure{A\arabic{figure}}
\renewcommand\thetable{A\arabic{table}}
\renewcommand\theHfigure{A\arabic{figure}}
\renewcommand\theHtable{A\arabic{table}}

\section{Overall Pipeline of CLeaR}
\label{appendix:overall_pipe}

We provide detailed pseudocode in Alg.~\ref{alg:overall_pipe} to facilitate a better understanding of our framework. The algorithm consists of three main stages: (1) computing orthogonal style anchors for multiple VFMs, (2) inverting these anchors into a pixel-space style image via ensemble optimization, and (3) generating the final output with energy-guided calibration.

In Stage 1, we generate a content image \(I_c\) from the description \(c\) and compute orthogonal style features \(\mathbf{f}_s^{(k)}\) for each VFM. Stage 2 optimizes a pixel-space image \(I_a\) to match these features simultaneously, producing a style anchor \(I_a^*\). Stage 3 uses the IP-Adapter \(\Phi_{\text{IP}}\) to condition the diffusion model on \(I_a^*\) and the target content \(y\) (or \(I_y\)), while additionally applying an energy correction based on the same VFMs. The energy guidance is only applied during the semantic stage (e.g., intermediate timesteps) to balance efficiency and style fidelity. The final output \(I_{\text{out}}\) is the decoded latent \(\mathbf{z}_0\). Hyperparameters such as \(\lambda_k\), \(\lambda_{\mathcal{R}}\), \(\eta\), \(\rho\) are set empirically.

\begin{algorithm}[H]
\caption{Overall Pipeline of CLeaR}
\label{alg:overall_pipe}
\begin{algorithmic}[1]
\REQUIRE Style image \(I\), content description \(c\), target content prompt \(y\) (or target image \(I_y\)), pre-trained VFMs \(\{F_k\}_{k=1}^K\), text-to-image model \(\Phi\), IP-Adapter \(\Phi_{\text{IP}}\), VAE decoder \(D\), DDIM sampler with \(\boldsymbol{\epsilon}_\theta\)
\ENSURE Output image \(I_{\text{out}}\)

\STATE \textbf{--- Stage 1: orthogonal style anchors ---}
\STATE \(I_c \gets \Phi(c)\) \COMMENT{generate content image}
\FOR{each \(k=1\) to \(K\)}
    \STATE \(\mathbf{f}^{(k)} \gets F_k(I)\), \(\mathbf{f}_c^{(k)} \gets F_k(I_c)\)
    \STATE \(\mathbf{f}_s^{(k)} \gets \mathbf{f}^{(k)} - \mathbf{f}_c^{(k)}\)
    \STATE \(\mathbf{f}_s^{(k)} \gets \mathbf{f}_s^{(k)} - \frac{\langle \mathbf{f}_s^{(k)}, \mathbf{f}_c^{(k)} \rangle}{\|\mathbf{f}_c^{(k)}\|^2} \mathbf{f}_c^{(k)}\) \COMMENT{orthogonal projection}
\ENDFOR

\STATE \textbf{--- Stage 2: ensemble inversion to obtain style anchor image ---}
\STATE Initialize \(I_a \sim \mathcal{N}(0, \sigma^2)\) \COMMENT{random tensor}
\FOR{\(t = 1\) to \(T_{\text{inv}}\)}
    \FOR{each \(k=1\) to \(K\)}
        \STATE \(\mathbf{f}_a^{(k)} \gets F_k(I_a)\)
    \ENDFOR
    \STATE \(\mathcal{L}_{\text{inv}} \gets \sum_{k=1}^K \lambda_k \|\mathbf{f}_a^{(k)} - \mathbf{f}_s^{(k)}\|^2 + \lambda_{\mathcal{R}} \mathcal{R}(I_a)\)
    \STATE \(I_a \gets I_a - \eta \nabla_{I_a} \mathcal{L}_{\text{inv}}\)
\ENDFOR
\STATE \(I_a^* \gets I_a\)

\STATE \textbf{--- Stage 3: energy-guided calibration during generation ---}
\STATE Initialize \(\mathbf{z}_T \sim \mathcal{N}(0,\mathbf{I})\)
\FOR{\(t = T\) down to \(1\)}
    \IF{\(t \in \mathcal{T}_{\mathrm{cal}}\)}
        \STATE \(\boldsymbol{\epsilon}_t \gets 
        \boldsymbol{\epsilon}_\theta(\mathbf{z}_t,t,I_a^*,y)\)

        \STATE \(\hat{\mathbf{z}}_0
        \gets
        (\mathbf{z}_t-\sigma_t\boldsymbol{\epsilon}_t)/\alpha_t\)

        \STATE \(\hat{\mathbf{x}}_0 \gets D(\hat{\mathbf{z}}_0)\)

        \STATE \(\mathcal{E}_t
        \gets
        \sum_{k=1}^{K}
        \lambda_k
        \left\|
        F_k(\hat{\mathbf{x}}_0)
        -
        \mathbf{f}_s^{(k)}
        \right\|^2\)

        \STATE \(\tilde{\mathbf{z}}_t
        \gets
        \mathbf{z}_t-\rho\nabla_{\mathbf{z}_t}\mathcal{E}_t\)

        \STATE \(\mathbf{z}_{t-1}
        \gets
        \mathrm{DDIM}
        \big(
        \tilde{\mathbf{z}}_t,
        \boldsymbol{\epsilon}_\theta(\tilde{\mathbf{z}}_t,t,I_a^*,y),
        t
        \big)\)
    \ELSE
        \STATE \(\mathbf{z}_{t-1}
        \gets
        \mathrm{DDIM}
        \big(
        \mathbf{z}_t,
        \boldsymbol{\epsilon}_\theta(\mathbf{z}_t,t,I_a^*,y),
        t
        \big)\)
    \ENDIF
\ENDFOR
\STATE \(I_{\mathrm{out}} \gets D(\mathbf{z}_0)\)
\RETURN \(I_{\mathrm{out}}\)
\end{algorithmic}
\end{algorithm}

\section{Proof of the Multi-Model Inversion Error Bound}
\label{appendix:proof_multi_model}

We provide a local first-order analysis of ensemble inversion.

\paragraph{Setup} Let
\begin{equation}
\hat x_K
=
\arg\min_x
\left\{
\frac{1}{K}\sum_{k=1}^K \|F_k(x)-s_k\|^2
+
\lambda \|x-x^\star\|^2
\right\},
\label{eq:app_obj}
\end{equation}
where $x^\star \in \mathbb{R}^d$ is the ideal content-free style anchor.
Write
\begin{equation}
\hat\Delta_K := \hat x_K-x^\star.
\label{eq:app_delta_def}
\end{equation}

\begin{assumption}[Local first-order regime]
\label{ass:local_regime}
Each $F_k$ is differentiable in a neighborhood of $x^\star$, and for sufficiently small $\Delta$,
\begin{equation}
F_k(x^\star+\Delta)
=
F_k(x^\star)+J_k\Delta+r_k(\Delta),
\qquad
J_k := \nabla F_k(x^\star),
\label{eq:app_taylor}
\end{equation}
where the remainder $r_k(\Delta)$ is lower-order than the linear term as $\|\Delta\|\to 0$.
\end{assumption}

\begin{assumption}[Ensemble-average bias cancellation]
\label{ass:bias_cancel}
The ensemble-averaged image-space bias vanishes:
\begin{equation}
\frac{1}{K}\sum_{k=1}^K J_k^\top \mu_k = 0,
\qquad
\mu_k := \mathbb{E}[\varepsilon_k].
\label{eq:app_bias_cancel}
\end{equation}
\end{assumption}

\begin{assumption}[Bounded Jacobians]
\label{ass:jacobian_bound}
The local Jacobians are uniformly bounded:
\begin{equation}
\|J_k\|_{\mathrm{op}}^2 \le M_k.
\label{eq:app_Mk}
\end{equation}
\end{assumption}

\begin{proposition}
\label{prop:appendix_main}
Under Assumption~\ref{ass:main_div_ident} and Appendix Assumptions~\ref{ass:local_regime}-\ref{ass:jacobian_bound},
\begin{equation}
\mathbb{E}\|\hat x_K-x^\star\|^2
\lesssim
\frac{d\,\bar M_K\,\sigma^2}{(m_K+\lambda)^2}
\left(
\rho_K+\frac{1-\rho_K}{K}
\right),
\label{eq:app_final}
\end{equation}
where
\begin{equation}
\bar M_K := \frac{1}{K}\sum_{k=1}^K M_k.
\label{eq:app_Mbar}
\end{equation}
\end{proposition}

\begin{proof}
Under Assumption~\ref{ass:local_regime}, the leading behavior is governed by the linearized objective
\begin{equation}
\tilde\Delta_K
=
\arg\min_\Delta
\left\{
\frac{1}{K}\sum_{k=1}^K \|J_k\Delta-\varepsilon_k\|^2
+
\lambda \|\Delta\|^2
\right\}.
\label{eq:app_lin_obj}
\end{equation}
Expanding the quadratic objective gives
\begin{align}
\frac{1}{K}\sum_{k=1}^K \|J_k\Delta-\varepsilon_k\|^2 + \lambda \|\Delta\|^2
&=
\frac{1}{K}\sum_{k=1}^K
\left(
\Delta^\top J_k^\top J_k \Delta
-2 \varepsilon_k^\top J_k \Delta
+\varepsilon_k^\top\varepsilon_k
\right)
+\lambda \Delta^\top\Delta \notag\\
&=
\Delta^\top H_K \Delta
-2 b_K^\top \Delta
+\lambda \Delta^\top\Delta
+\text{constant},
\label{eq:app_expand}
\end{align}
where
\begin{equation}
b_K := \frac{1}{K}\sum_{k=1}^K J_k^\top \varepsilon_k.
\label{eq:app_bk}
\end{equation}
Differentiating Eq.~\ref{eq:app_expand} with respect to $\Delta$ and setting the gradient to zero yields
\begin{equation}
(H_K+\lambda I)\tilde\Delta_K = b_K.
\label{eq:app_normal}
\end{equation}
Hence
\begin{equation}
\tilde\Delta_K = (H_K+\lambda I)^{-1} b_K.
\label{eq:app_solution}
\end{equation}

By Assumption~\ref{ass:main_div_ident},
\begin{equation}
H_K+\lambda I \succeq (m_K+\lambda)I,
\end{equation}
so
\begin{equation}
\|(H_K+\lambda I)^{-1}\|_{\mathrm{op}}
\le
\frac{1}{m_K+\lambda}.
\label{eq:app_inv}
\end{equation}
Therefore,
\begin{equation}
\mathbb{E}\|\tilde\Delta_K\|^2
\le
\frac{1}{(m_K+\lambda)^2}\mathbb{E}\|b_K\|^2.
\label{eq:app_reduce}
\end{equation}

We now decompose $b_K$ into bias and centered fluctuation terms:
\begin{equation}
b_K
=
\frac{1}{K}\sum_{k=1}^K J_k^\top \mu_k
+
\frac{1}{K}\sum_{k=1}^K J_k^\top \tilde\varepsilon_k
=: \beta_K + u_K.
\label{eq:app_decomp}
\end{equation}
By Assumption~\ref{ass:bias_cancel},
\begin{equation}
\beta_K = 0.
\label{eq:app_beta_zero}
\end{equation}
Hence
\begin{equation}
b_K = u_K
\qquad\text{and}\qquad
\mathbb{E}\|b_K\|^2 = \mathbb{E}\|u_K\|^2.
\label{eq:app_reduce_u}
\end{equation}

Now
\begin{equation}
u_K
=
\frac{1}{K}\sum_{k=1}^K J_k^\top \tilde\varepsilon_k,
\end{equation}
so
\begin{equation}
\mathbb{E}\|u_K\|^2
=
\frac{1}{K^2}
\sum_{k=1}^K \sum_{j=1}^K
\mathbb{E}\!\left[
\tilde\varepsilon_k^\top J_k J_j^\top \tilde\varepsilon_j
\right].
\label{eq:app_u_expand}
\end{equation}

For the diagonal terms,
\begin{align}
\mathbb{E}\!\left[
\tilde\varepsilon_k^\top J_k J_k^\top \tilde\varepsilon_k
\right]
&=
\mathrm{tr}\!\left(
J_k J_k^\top \mathbb{E}[\tilde\varepsilon_k\tilde\varepsilon_k^\top]
\right) \notag\\
&\le
d\,\|J_k\|_{\mathrm{op}}^2\,\sigma^2 \notag\\
&\le
d\,M_k\,\sigma^2.
\label{eq:app_diag}
\end{align}

For the off-diagonal terms \(k\neq j\), Assumption~\ref{ass:main_div_ident} and Assumption~\ref{ass:jacobian_bound} give
\begin{align}
\left|
\mathbb{E}\!\left[
\tilde\varepsilon_k^\top J_k J_j^\top \tilde\varepsilon_j
\right]
\right|
&=
\left|
\mathrm{tr}\!\left(
J_k J_j^\top \mathbb{E}[\tilde\varepsilon_j\tilde\varepsilon_k^\top]
\right)
\right| \notag\\
&\le
d\,\|J_k\|_{\mathrm{op}}\|J_j\|_{\mathrm{op}}\,\rho_K \sigma^2 \notag\\
&\le
d\,(M_kM_j)^{1/2}\rho_K \sigma^2 \notag\\
&\le
\frac{d\,\rho_K\,\sigma^2}{2}(M_k+M_j).
\label{eq:app_offdiag}
\end{align}

Substituting Eq.~\ref{eq:app_diag} and Eq.~\ref{eq:app_offdiag} into Eq.~\ref{eq:app_u_expand}, we obtain
\begin{align}
\mathbb{E}\|u_K\|^2
&\le
\frac{1}{K^2}
\left[
\sum_{k=1}^K d\,M_k\,\sigma^2
+
\sum_{k\neq j}\frac{d\,\rho_K\,\sigma^2}{2}(M_k+M_j)
\right] \notag\\
&=
\frac{d\,\sigma^2}{K^2}
\left[
\sum_{k=1}^K M_k
+
\rho_K (K-1)\sum_{k=1}^K M_k
\right] \notag\\
&=
d\,\bar M_K\,\sigma^2
\left(
\frac{1}{K}+\rho_K\frac{K-1}{K}
\right) \notag\\
&=
d\,\bar M_K\,\sigma^2
\left(
\rho_K+\frac{1-\rho_K}{K}
\right).
\label{eq:app_u_bound}
\end{align}

Combining Eq.~\ref{eq:app_reduce}, Eq.~\ref{eq:app_reduce_u}, and Eq.~\ref{eq:app_u_bound} yields
\begin{equation}
\mathbb{E}\|\tilde\Delta_K\|^2
\le
\frac{d\,\bar M_K\,\sigma^2}{(m_K+\lambda)^2}
\left(
\rho_K+\frac{1-\rho_K}{K}
\right).
\label{eq:app_linear_final}
\end{equation}
Under Assumption~\ref{ass:local_regime}, the neglected remainder is lower-order in the local regime, so the same leading-order bound applies to the exact estimator $\hat x_K-x^\star$.
\end{proof}

\begin{figure}[!htbp]
    \centering
    \includegraphics[width=\linewidth]{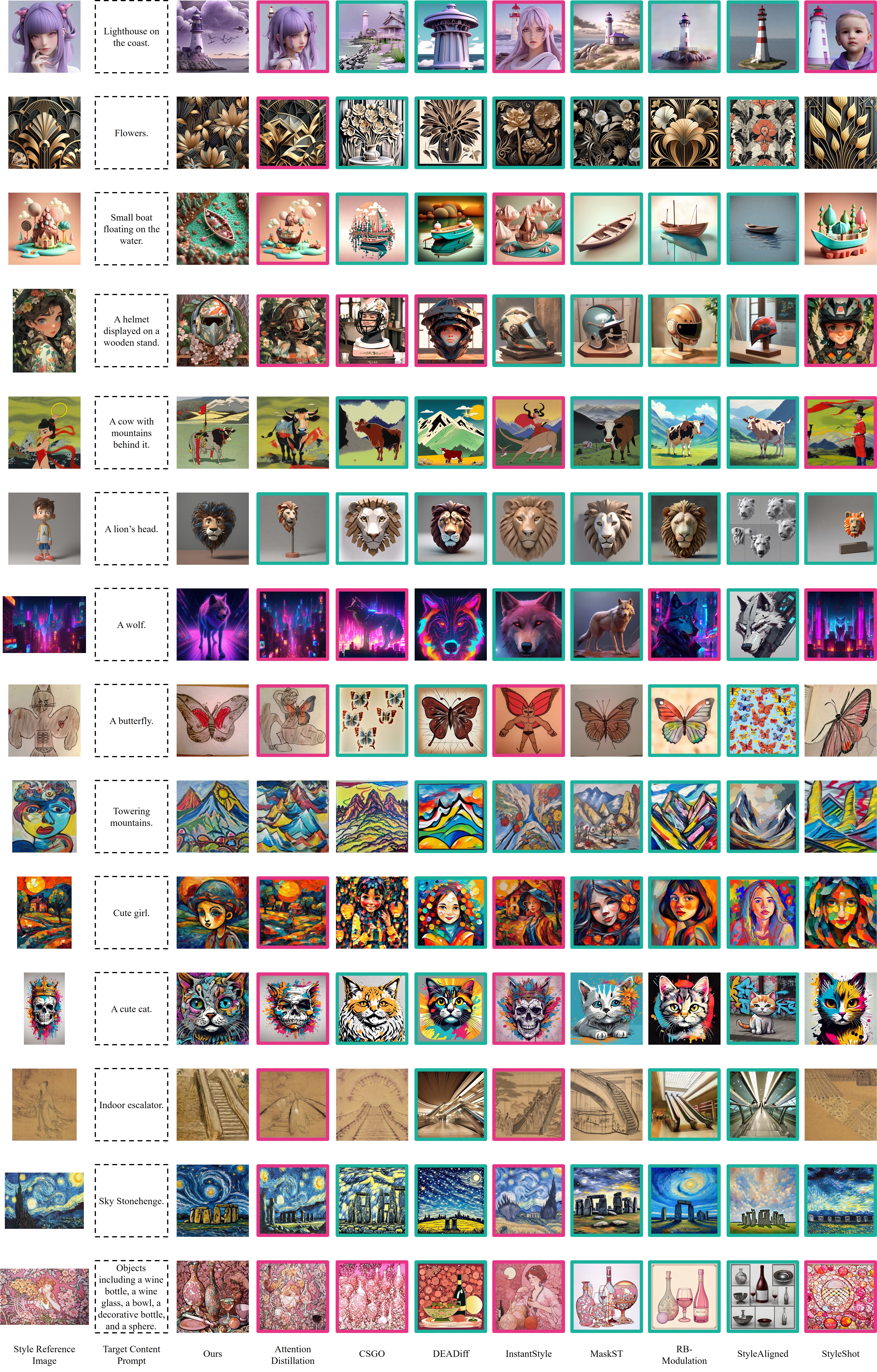}
    \vspace{-1.3em}
    \caption{Extended qualitative comparisons, with boxes highlighting \textcolor{my_magenta} {content leakage} and \textcolor{my_cyan} {style degradation}. Our method maintains balanced stylization across all examples.}
    \label{fig:qual_comp_plus}
\end{figure}

\section{More Qualitative Comparison Results}
\label{appendix:qual_comp}

Fig.~\ref{fig:qual_comp_plus} provides additional qualitative comparisons. Consistent with the observations in Sec.~\ref{sec:comparison}, existing methods still suffer from either content leakage (magenta boxes) or style degradation (cyan boxes) across diverse examples, while our CLeaR achieves clean stylization without compromising either aspect.

\section{More Examples of Style Anchors}
\label{appendix:style_anchor}

Fig.~\ref{fig:style_anchor_plus} presents additional style anchors extracted by our Ensemble Inversion. Across various artistic styles, the anchors consistently suppress content semantics (e.g., objects, faces, text) while preserving style attributes such as color, texture, and brushstrokes.

\begin{figure}[!htbp]
    \centering
    \includegraphics[width=\linewidth]{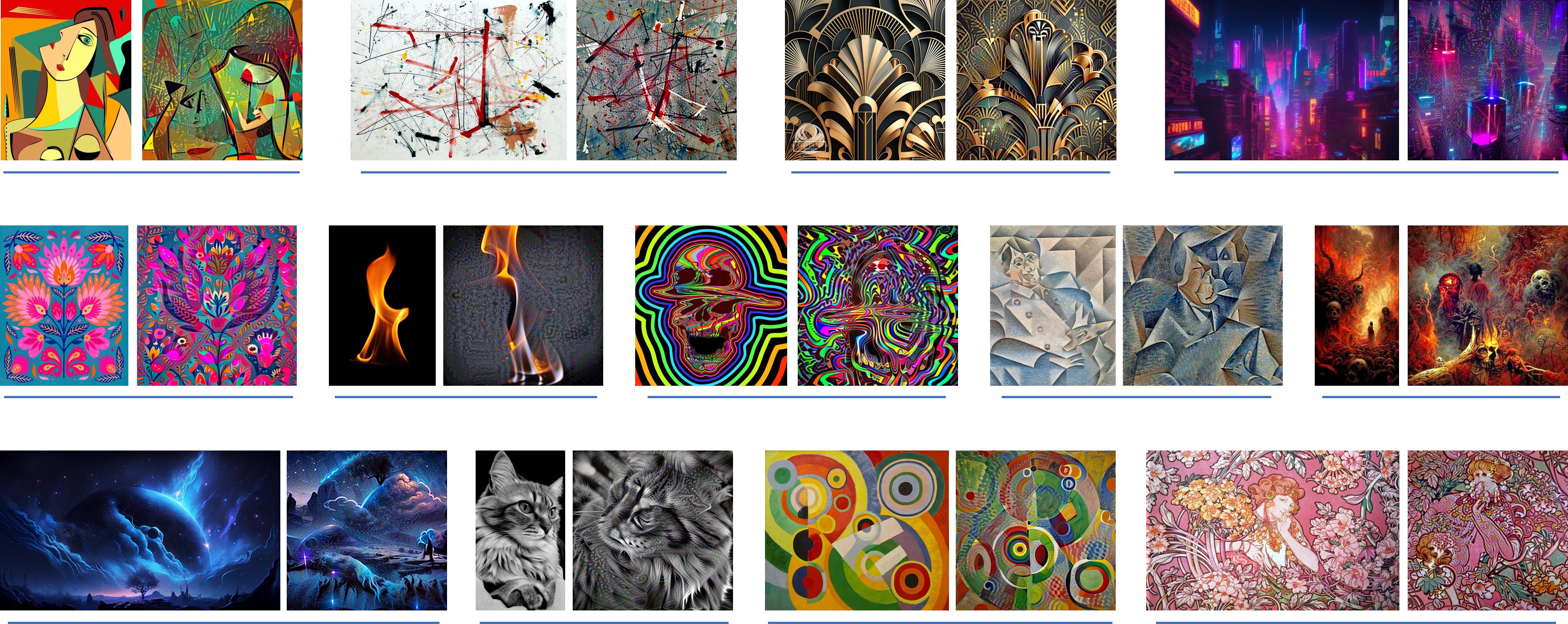}
    \vspace{-1.3em}
    \caption{Additional style anchor examples. Each pair shows the original reference (left) and the extracted anchor (right).}
    \label{fig:style_anchor_plus}
\end{figure}

\section{Comparisons on Held-Out VFMs and AI Evaluator}
\label{appendix:heldout_quant_comp}

\begin{table}[!htbp]
    \fontsize{12pt}{17pt}\selectfont
    \centering
    \caption{Performance of held-out VFMs and independent AI evaluators, with the \textbf{best} and \underline{second-best} scores marked accordingly.}
    \label{tab:heldout_quant_comp}
    \begin{adjustbox}{width=\linewidth}
    \begin{tabular}{ccccccccccc}
        \hline
        \multirow{2}{*}{} & \multicolumn{2}{c}{Style Alignment} & \multicolumn{2}{c}{Content Alignment} & \multicolumn{4}{c}{Content Leakage} & \multicolumn{2}{c}{Aesthetic Quality} \\
        \cmidrule(lr){2-3} \cmidrule(lr){4-5} \cmidrule(lr){6-9} \cmidrule(lr){10-11}
        & SigLIP$\uparrow$ & Nomic$\uparrow$ & SigLIP$\uparrow$ & Nomic$\uparrow$ & SigLIP$\downarrow$ & Nomic$\downarrow$ & Kimi-K2.6$\uparrow$ & GPT-5.5$\uparrow$ & Kimi-K2.6$\uparrow$ & GPT-5.5$\uparrow$ \\
        \hline
        Attention Distillation & 0.669 & 0.842 & 0.0867 & 0.0573 & 0.588 & 0.797 & 2.03 & 1.68 & 3.20 & 3.10 \\
        CSGO & 0.592 & 0.800 & 0.117 & 0.0777 & 0.539 & 0.774 & 3.68 & 3.27 & 3.43 & 3.36 \\
        DEADiff & 0.563 & 0.788 & 0.126 & 0.0812 & 0.552 & 0.761 & 3.71 & 3.24 & 3.62 & 3.58 \\
        InstantStyle & 0.609 & 0.800 & 0.123 & 0.0826 & 0.529 & 0.760 & 2.91 & 2.47 & 3.89 & 3.86 \\
        MaskST & 0.588 & 0.790 & 0.128 & 0.0833 & \textbf{0.521} & 0.753 & 3.56 & 3.08 & 4.03 & 4.01 \\
        RB-Modulation & 0.584 & 0.789 & 0.136 & \underline{0.0887} & 0.557 & 0.761 & 3.92 & 3.53 & \textbf{4.48} & \underline{4.39} \\
        StyleAligned & 0.544 & 0.772 & \textbf{0.137} & \textbf{0.0893} & 0.531 & \underline{0.751} & \underline{4.03} & \underline{3.58} & \underline{4.39} & 4.18 \\
        StyleShot & 0.612 & 0.814 & 0.0997 & 0.0698 & 0.534 & 0.771 & 3.79 & 3.41 & 3.78 & 3.74 \\
        Ours & \textbf{0.754} & \textbf{0.876} & \textbf{0.137} & 0.0803 & \underline{0.527} & \textbf{0.748} & \textbf{4.28} & \textbf{3.92} & 4.20 & \textbf{4.52} \\
        \hline
    \end{tabular}
    \end{adjustbox}
\end{table}

To address the concern that CLeaR optimizes and evaluates on overlapping VFM features (CSD-CLIP, VGG, DINO) and that Qwen3 is used both for content description generation and AI-based evaluation, we conduct additional experiments with held-out VFMs and independent AI evaluators. We add SigLIP \cite{tschannen2025siglip} and Nomic \cite{nussbaum2024nomic} as additional VFMs. Since our original selection (e.g., CLIP, DINO, CSD-CLIP) already covers the most standard models for style transfer and vision-language representation, few suitable alternatives remain. For AI evaluation, we include Kimi-K2.6 \cite{team2025kimi} and GPT-5.5 \cite{openai2026gpt55}. To ensure comprehensive coverage, we evaluate these new metrics across all four aspects.

As shown in Tab.~\ref{tab:heldout_quant_comp}, CLeaR consistently achieves the best or competitive performance across all held-out VFMs and independent AI evaluators. These results confirm that our improvements generalize beyond the models used during optimization and are not biased by the closed-loop evaluation setup.

\section{Ablation on Style Categories}
\label{appendix:style_abl}

To investigate performance variation across different styles, we compute SA and CL scores for all 73 style types in StyleBench.

As seen in Tab.~\ref{tab:style_abl}, styles whose stylistic identity is carried by rendering rather than subject matter (e.g., Impressionism, Watercolor, Line Art, and Primitivism) achieve high scores on both SA and CL, indicating easier C-S disentanglement when stylistic signals are distributed across the entire image. In contrast, styles associated with religious, mythological, or historical subjects (e.g., Classicism, Rococo, and Baroque) exhibit lower CL scores even when SA remains high. Such themes often involve complex figures, narratives, or symbols that are hard to enumerate in \(c\), making C-S disentanglement more challenging during Ensemble Inversion.

\begin{table*}[!htbp]
    \fontsize{12pt}{17pt}\selectfont
    \centering
    \caption{Breakdown across the 73 categories on StyleBench. Rich-texture styles yield high SA and CL, while subject-centric styles show lower CL despite strong SA.}
    \vspace{0.4em}
    \label{tab:style_abl}
    \begin{adjustbox}{width=\linewidth}
    \begin{tabular}{ccccccccc}
        \hline
        Style & SA$\uparrow$ & CL$\uparrow$ & Style & SA$\uparrow$ & CL$\uparrow$ & Style & SA$\uparrow$ & CL$\uparrow$ \\
        \hline
        3D Model & 0.792 & 0.276 & Expressionist & 0.855 & 0.476 & Origami & 0.815 & 0.528 \\
        3D Model 01 & 0.729 & 0.457 & Fantasy Art & 0.975 & 0.453 & Orphism & 0.853 & 0.604 \\
        3D Model 02 & 0.764 & 0.511 & Fauvism & 0.926 & 0.339 & Others & 0.799 & 0.380 \\
        3D Model 03 & 0.703 & 0.431 & Flat Vector & 0.585 & 0.338 & Photographic & 0.741 & 0.311 \\
        3D Model 04 & 0.802 & 0.462 & Folk Art & 0.498 & 0.320 & Pixel Art & 0.674 & 0.420 \\
        3D Model 05 & 0.726 & 0.433 & Gongbi & 0.908 & 0.473 & Pointilism & 0.855 & 0.139 \\
        Abstract & 0.741 & 0.496 & Graffiti & 0.640 & 0.279 & Pop Art & 0.671 & 0.479 \\
        Abstract 01 & 0.845 & 0.498 & Hyperrealism & 0.808 & 0.299 & Post-Impressionism & 0.972 & 0.466 \\
        Analog Film & 0.635 & 0.382 & Icon & 0.476 & 0.605 & Precisionism & 0.893 & 0.491 \\
        Anime & 0.837 & 0.394 & Icon 01 & 0.517 & 0.551 & Primitivism & 0.883 & 0.661 \\
        Anime 01 & 0.543 & 0.459 & Icon 02 & 0.551 & 0.432 & Psychedelic & 0.434 & 0.402 \\
        Anime 02 & 0.730 & 0.422 & Impressionism & 0.967 & 0.921 & Realism & 0.924 & 0.547 \\
        Anime 03 & 0.866 & 0.546 & Ink and Wash Painting & 0.864 & 0.246 & Rococo & 0.868 & 0.186 \\
        Anime 04 & 0.823 & 0.458 & IsoMetric & 0.708 & 0.442 & Smoke\&Light & 0.820 & 0.569 \\
        Anime 05 & 0.907 & 0.358 & Japonism & 0.743 & 0.217 & Statue & 0.685 & 0.497 \\
        Anime 06 & 0.878 & 0.631 & Line Art & 0.803 & 0.675 & Steampunk & 0.665 & 0.264 \\
        Anime 07 & 0.817 & 0.389 & Low Poly & 0.780 & 0.546 & Stick Figure & 0.697 & 0.498 \\
        Art Deco & 0.605 & 0.278 & Luminism & 0.992 & 0.529 & Stickers & 0.580 & 0.388 \\
        Baroque & 0.901 & 0.266 & Macabre & 0.937 & 0.584 & Surrealist & 0.890 & 0.507 \\
        Children's Painting & 0.755 & 0.554 & MineCraft & 0.864 & 0.459 & Symbolism & 0.892 & 0.441 \\
        Classicsm & 0.882 & 0.0184 & Monochrome & 0.774 & 0.467 & Tonalism & 0.887 & 0.492 \\
        Constructivism & 0.710 & 0.506 & Neo-Figurative Art & 0.922 & 0.541 & Typography & 0.200 & 0.440 \\
        Craft Clay & 0.475 & 0.508 & Neoclassicism & 0.813 & 0.286 & Watercolor & 0.858 & 0.667 \\
        Cublism & 0.930 & 0.369 & Nouveau & 0.920 & 0.395 & & & \\
        Cyberpunk & 0.923 & 0.499 & Op Art & 0.672 & 0.538 & & & \\
        \hline
    \end{tabular}
    \end{adjustbox}
\end{table*}

\section{Analysis on Failure Modes}
\label{appendix:failure_mode}

We analyze three representative failure modes of CLeaR. All stem from limitations in current pretrained modules rather than the disentanglement formulation itself.

\paragraph{Incomplete Content Description} In theory, an ideal \(c\) allows OSP to remove all content-related components. However, as discussed in Appendix~\ref{appendix:style_abl}, certain styles are difficult to describe clearly, leading to residual content leakage.

\paragraph{Degenerate Style Anchor} Fig.~\ref{fig:style_anchor_bad} shows cases where Ensemble Inversion produces a degenerate anchor that either contains almost no information (col 2) or replicates the style reference (cols 1, 3, 4). This failure mode is frequently observed in styles such as Line Art, Stick Figure, and Icon. Although orthogonal projection is designed to remove content from the style target, current VFM embeddings do not always provide a direction that separates style from content \cite{yao2026learning}. When style is tightly coupled with a specific object, the inversion objective is satisfied either by an anchor without meaningful structure or by one that copies the reference.

\begin{figure}[!htbp]
    \centering
    \includegraphics[width=1\linewidth]{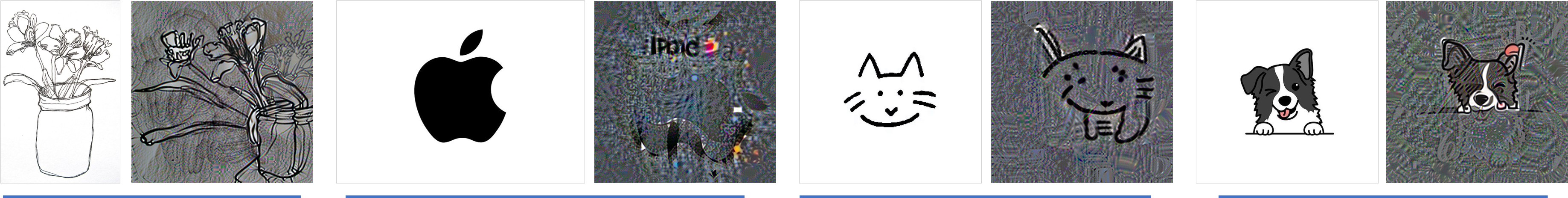}
    \vspace{-1.3em}
    \caption{Degenerate style anchors. Ensemble Inversion may produce an anchor with no structure or one that copies the reference, especially when content and style are tightly coupled.}
    \label{fig:style_anchor_bad}
\end{figure}

\paragraph{Generation Artifacts} Fig.~\ref{fig:qual_bad} shows another failure mode that arises when the style anchor \(I_a^*\) is used as the image condition for diffusion generation \cite{zhou2024twindiffusion,zhang2025multi,zhou2025panollama}, where the output contains noisy lines and blobs. As discussed in Sec.~\ref{sec:style_calib}, the inverted anchor lies outside the adapter's training distribution, so the adapter may misinterpret noise in the anchor as semantic content, forming artifacts in the final result.

\begin{figure}[!htbp]
    \centering
    \includegraphics[width=0.8\linewidth]{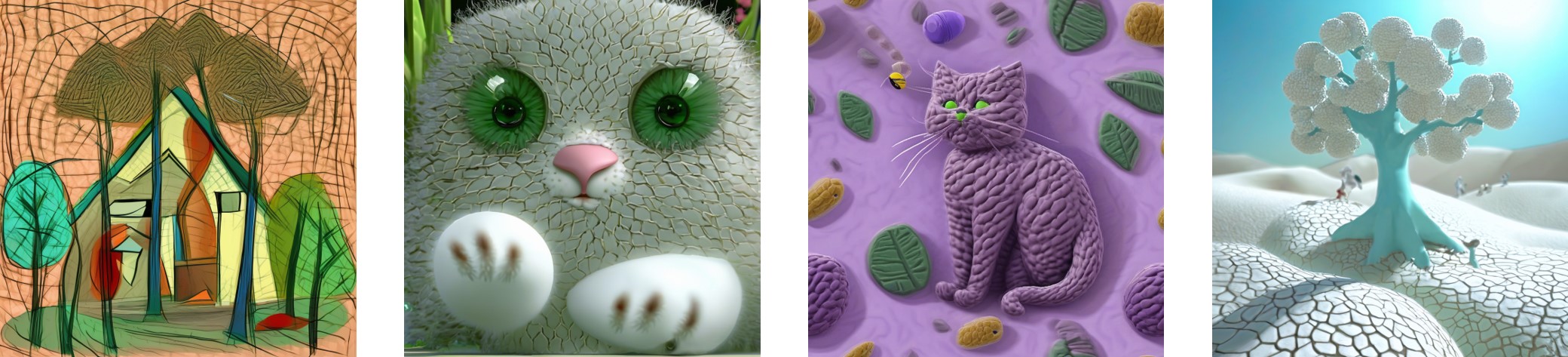}
    \vspace{-0.2em}
    \caption{Artifacts in the generation stage. The adapter may misinterpret inversion noise in style anchor as semantic content, producing noisy lines and blobs.}
    \label{fig:qual_bad}
\end{figure}


\end{document}